\documentclass[10pt,twocolumn,conference]{IEEEtran}
\usepackage{amssymb}
\usepackage[reqno]{amsmath}
\usepackage{amsthm,color}
\usepackage{amsfonts}
\usepackage{subfigure}
\usepackage{times}
\usepackage[final]{graphicx}
\usepackage{times,amsmath,epsfig}
\usepackage{latexsym,amssymb,mathrsfs}
\usepackage{cite}
\usepackage[ruled]{algorithm2e}
\usepackage{setspace}
\usepackage{color}
\usepackage{comment}
\usepackage{url}
\usepackage{algorithmic}

\newtheorem{lem}{Lemma}

\newtheorem{theorem}{Theorem}
\newtheorem{corollary}{Corollary}
\usepackage{comment}
\newcommand{\beq}{\begin{equation}}
\newcommand{\eeq}{\end{equation}}
\newcommand{\bea}{\begin{eqnarray}}
\newcommand{\eea}{\end{eqnarray}}
\newcommand{\bef}{\begin{figure}}
	\newcommand{\eef}{\end{figure}}
\newcommand{\bsc}{\begin{scriptstyle}}
	\newcommand{\esc}{\end{scriptstyle}}
\newcommand{\bd}{\begin{displaymath}}
\newcommand{\ed}{\end{displaymath}}

\IEEEoverridecommandlockouts

\begin{document}


\title{Action-Manipulation Attacks Against Stochastic Bandits: Attacks and Defense }
\author{Guanlin Liu and Lifeng Lai\thanks{G. Liu and L. Lai are with Department of Electrical and Computer Engineering, University of California, Davis, CA, 95616. Email: \{glnliu,lflai\}@ucdavis.edu. The work of G. Liu and L. Lai was supported by National Science Foundation under Grants CCF-1717943, ECCS-1711468, CNS-1824553 and CCF-1908258. This paper will be presented in part in the 2020 IEEE International Conference on Acoustics, Speech and Signal Processing~\cite{Icasspour}.}}
\date{\today}
\maketitle
\pagestyle{plain}
\begin{abstract}
Due to the broad range of applications of stochastic multi-armed bandit model, understanding the effects of adversarial attacks and designing bandit algorithms robust to attacks are essential for the safe applications of this model. In this paper, we introduce a new class of attack named action-manipulation attack. In this attack, an adversary can change the action signal selected by the user. We show that without knowledge of mean rewards of arms, our proposed attack can manipulate Upper Confidence Bound (UCB) algorithm, a widely used bandit algorithm, into pulling a target arm very frequently by spending only logarithmic cost. To defend against this class of attacks, we introduce a novel algorithm that is robust to action-manipulation attacks when an upper bound for the total attack cost is given. We prove that our algorithm has a pseudo-regret upper bounded by $\mathcal{O}(\max\{\log T,A\})$, where $T$ is the total number of rounds and $A$ is the upper bound of the total attack cost.
	  	
\end{abstract}

\begin{IEEEkeywords} 
	Stochastic bandits, action-manipulation attack, UCB.
\end{IEEEkeywords}

\section{Introduction}
In order to develop trustworthy machine learning systems, understanding adversarial attacks on learning systems and correspondingly building robust defense mechanisms have attracted significant recent research interests~\cite{goodfellow2014explaining, huang2017adversarial, lin2017tactics,mei2015using, biggio2012poisoning, xiao2015feature, li2016data, alfeld2016data}. In this paper, we focus on multiple armed bandits (MABs), a simple but very powerful framework of online learning that makes decisions over time under uncertainty. MABs problem is widely investigated in machine learning and signal processing~\cite{5378483,7103356,8777299,5398950,5535151,7931644,6200864} and has many applicants in a variety of scenarios such as displaying advertisements \cite{Chapelle:2014:SSR:2699158.2532128}, articles recommendation \cite{Li:2010:CAP:1772690.1772758}, cognitive radios~\cite{Lai:TMC:11,8682263} and search engines~\cite{kveton2015cascading}, to name a few. 

Of particular relevance to our work is a line of interesting recent work on online reward-manipulation attacks on stochastic MABs~\cite{jun2018adversarial,liu2019data,Guan:AAAI:20,lykouris2018stochastic}. In the reward-manipulation attacks, there is an adversary who can change the reward signal from the environment, and hence the reward signal received by the user is not the true reward signal from the environment. In particular,~\cite{jun2018adversarial} proposes an interesting attack strategy that can force a user, who runs either $\epsilon$-Greedy and or Upper Confidence Bound (UCB) algorithm, to select a target arm while only spending effort that grows in logarithmic order.
\cite{liu2019data} proposes an optimization based framework for offline reward-manipulation attacks. Furthermore, it studies a form of online attack strategy that is effective in attacking any bandit algorithm that has a regret scaling in logarithm order, without knowing  what particular algorithm the user is using. \cite{Guan:AAAI:20} considers an attack model where an adversary attacks with a
certain probability at each round but its attack value can be
arbitrary and unbounded. The paper proposes algorithms that are robust to these types of attacks. \cite{lykouris2018stochastic} considers how to defend against reward-manipulation attacks, a complementary problem to ~\cite{jun2018adversarial,liu2019data}. In particular, \cite{lykouris2018stochastic} introduces a bandit algorithm that is robust to reward-manipulation attacks under certain attack cost, by using a multi-layer approach. 
\cite{DBLP:journals/corr/abs-1906-01528} introduces another model of adversary setting where each arm is able to manipulate its own reward and seeks to maximize its own expected number of pull count. Under this setting, \cite{DBLP:journals/corr/abs-1906-01528} analyzes the robustness of Thompson Sampling, UCB, and $\epsilon$-greedy with the adversary, and proves that all three algorithms achieve a regret upper bound that increases over rounds in a logarithmic order or increases with attack cost in a linear order. 
This line of reward-manipulation attack has also recently been investigated for contextual bandits in~\cite{DBLP:journals/corr/abs-1808-05760}, which develops an attack algorithm that can force the bandit algorithm to pull a target arm for a target contextual vector by slightly manipulating rewards in the data.
 
In this paper, we introduce a new class of attacks on MABs named action-manipulation attack. In the action-manipulation attack, an attacker, sitting between the environment and the user, can change the action selected by the user to another action. The user will then receive a reward from the environment corresponding to the action chosen by the attacker. Compared with the reward-manipulation attacks discussed above, the action-manipulation attack is more difficult to carry out. In particular, as the action-manipulation attack only changes the action, it can impact but does not have direct control of the reward signal, because the reward signal will be a random variable drawn from a distribution depending on the action chosen by the attacker. This is in contrast to reward-manipulation attacks where an attacker has direct control and can change the reward signal to any value.

In order to demonstrate the significant security threat of action-manipulation attacks to stochastic bandits, we propose an action-manipulation attack strategy against the widely used UCB algorithm. The proposed attack strategy aims to force the user to pull a target arm chosen by the attacker frequently. We assume that the attacker does not know the true mean reward of each arm. 
The assumption that the attacker does not know the mean rewards of arms is necessary, as otherwise the attacker can perform the attack trivially. To see this, with the knowledge of the mean rewards, the attacker knows which arm has the worst mean reward and can perform the following oracle attack: when the user pulls a non-target arm, the attacker change the arm to the worst arm. This oracle attack makes all non-target arms have expected rewards less than that of the target arm, if the target arm selected by the attacker is not the worst arm. In addition, under this attack, all sublinear-regret bandit algorithms will pull the target arm $\mathcal{O}(T)$ times. However, the oracle attack is not practical. The goal of our work is to develop an attack strategy that has similar performance of the oracle attack strategy without requiring the knowledge of the true mean rewards. When the user pulls a non-target arm, the attacker could decide to attack by changing the action to the possible worst arm. As the attacker does not know the true value of arms, our attack scheme relies on lower confidence bounds (LCB) of the value of each arm in making attack decision. Correspondingly, we name our attack scheme as LCB attack strategy. Our analysis shows that, if the target arm selected by the attacker is not the worst arm, the LCB attack strategy can successfully manipulate the user to select the target arm almost all the time with an only logarithmic cost. In particular, LCB attack strategy can force the user to pull the target arm $T-\mathcal{O}(\log(T))$ times over $T$ rounds, with total attack cost being only $\mathcal{O}(\log(T))$. On the other hand, we also show that, if the target arm is the worst arm and the attacker can only incur logarithmic costs, no attack algorithm can force the user to pull the worst arm more than $T-\mathcal{O}(T^\alpha)$ times. 

Motivated by the analysis of the action-manipulation attacks and the significant security threat to MABs, we then design a bandit algorithm which can defend against the action-manipulation attacks and still is able to achieve a small regret. 
The main idea of the proposed algorithm is to bound the maximum amount of offset, in terms of user's estimate of the mean rewards, that can be introduced by the action-manipulation attacks. We then use this estimate of maximum offset to properly modify the UCB algorithm and build specially designed high-probability upper bounds of the mean rewards so as to decide which arm to pull. We name our bandit algorithm as maximum offset upper confidence bound (MOUCB). In particular, our algorithm firstly pulls every arm a certain of times and then pulls the arm whose modified upper confidence bound is largest. 
Furthermore, we prove that MOUCB bandit algorithm has a pseudo-regret upper bounded by $\mathcal{O}(\max\{\log T,A\})$, where $T$ is the total number of rounds and $A$ is an upper bound for the total attack cost. In particular, if $A$ scales as $\log(T)$, MOUCB archives a logarithm pseudo-regret which is same as the regret of UCB algorithm.

The remainder of the paper is organized as follows. In Section \ref{sec2}, we describe the model. In Section~\ref{sec3}, we describe the LCB attack strategy and analyze its accumulative attack cost. In Section~\ref{sec:defend}, we propose MOUCB and analyze its regret. In Section \ref{sec4}, we provide numerical examples to validate the theoretic analysis. Finally, we offer several concluding remarks in Section \ref{sec5}. The proofs are collected in Appendix.

\section{Model}\label{sec2}
In this section, we introduce our model. We consider the standard multi-armed stochastic bandit problems setting. The environment consists of $K$ arms, with each arm corresponds to a fixed but unknown reward distribution. The bandit algorithm, which is also called ``user'' in this paper, proceeds in discrete time $t= 1, 2, \dots, T$, in which $T$ is the total number of rounds. At each round $t$, the user pulls an arm (or action) $I_t \in \{1, \dots, K\}$ and receives a random reward $r_t$ drawn from the reward distribution of arm $I_t$.  
Denote $\tau_i(t):= \{ s:s\le t,I_s=i \} $ as the set of rounds up to $t$ where the user chooses arm $i$, $N_i(t):=|\tau_i(t)|$ as the number of rounds that arm $i$ was pulled by the user up to time $t$ and \begin{equation}
\hat{\mu}_i(t):=N_i(t)^{-1}\sum_{s\in\tau_i(t)}r_s\label{eq:em}
\end{equation} as the empirical mean reward of arm $i$. 
The pseudo-regret $\bar{R}(T)$ is defined as
\begin{equation} \label{eq:pregertdefine}
\bar{R}(T)=T\max\limits_{\max_{i\in[K]}}\mu_i-\mathbb{E}\left[\sum_{t=1}^T r_t\right].
\end{equation}
The goal of the user to minimize $\bar{R}(T)$.


In this paper, we introduce a novel adversary setting, in which the attacker sits between the user and the environment. The attacker can monitor the actions of the user and the reward signals from the environment. Furthermore, the attacker can introduce action-manipulation attacks on stochastic bandits. In particular, at each round $t$, after the user chooses an arm $I_t$, the attacker can manipulate the user's action by changing $I_t$ to another $I_t^0 \in \{1, \dots, K\}$. If the attacker decides not to attack, $I_t^0=I_t$. The environment generates a random reward $r_t$ from the reward distribution of post-attack arm $I_t^0$. Then the user and the attacker receive reward $r_t$ from the environment. 
\begin{figure}[htbp]
	\begin{center}
		\includegraphics[width=0.45\textwidth]{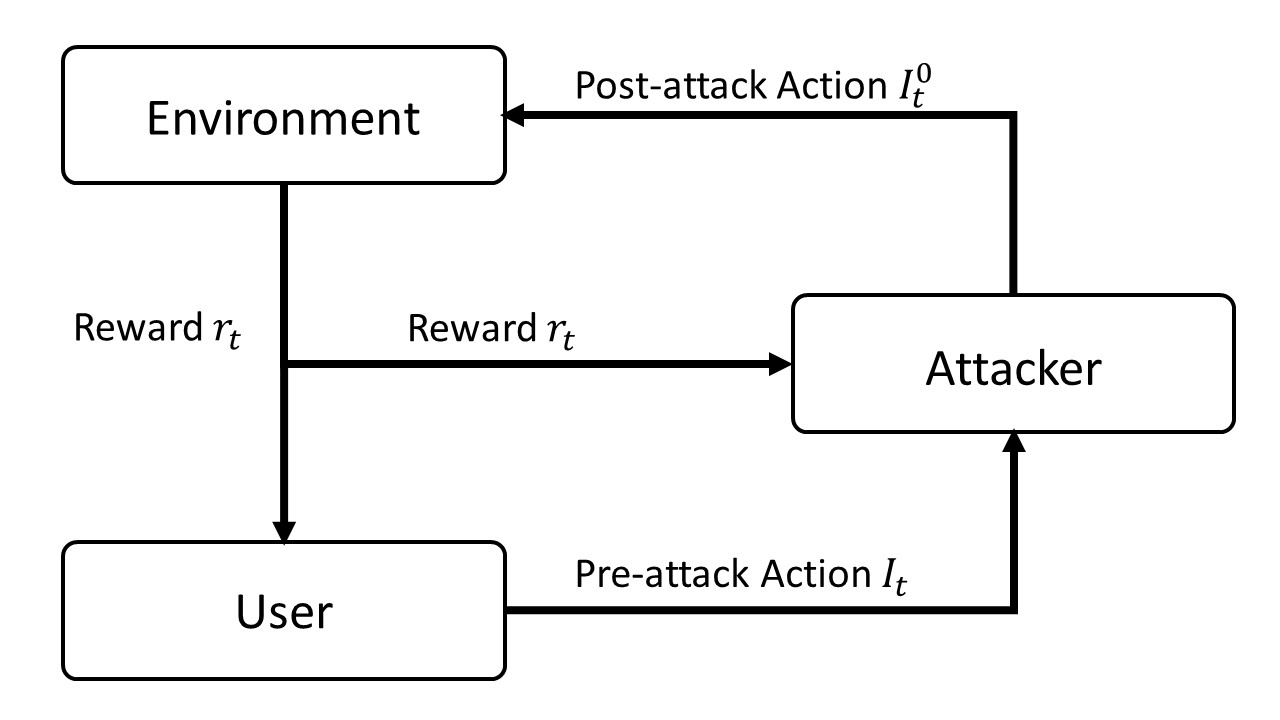}
		\caption{Action-manipulation attack model}
		\label{Model}
	\end{center}
\end{figure}

 Without loss of generality and for notation convenience, we assume arm $K$ is the ``attack target" arm or target arm. The attacker's goal is to manipulate the user into pulling the target arm very frequently but by making attacks as rarely as possible. Define the set of rounds when the attacker decides to attack as $\mathcal{C}:=\{ t: t\le T,I_t^0\not= I_t \}$. The cumulative attack cost is the total number of rounds where the attacker decides to attack, i.e., $|\mathcal{C}|$.

 In this paper, we assume that the reward distribution of arm $i$ follows $\sigma^2$-sub-Gaussian distributions with mean $\mu_i$. Denote the true reward vector as $\boldsymbol{\mu}=[\mu_1,\cdots,\mu_K]$. Neither the user nor the attacker knows $\boldsymbol{\mu}$, but $\sigma^2$ is known to both the user and the attacker. Define the difference of mean value of arm $i$ and $j$ as $\Delta_{i,j}=\mu_i-\mu_j$. Furthermore, we refer to the best arm as $i_O = \arg \max_i \mu_i$ and the worst arm as $i_W = \arg \min_i \mu_i$.

Note that the assumption that the attacker does not know $\boldsymbol{\mu}$ is important. If the attacker knows these values, the attacker can adopt a trivial oracle attack scheme: whenever the user pulls a non-target arm $I_t$, the attacker changes $I_t$ to the worst arm $i_W$. It is easy to show that, if the user uses a bandit algorithm that has a regret upper bounded of $\mathcal{O}(\log(T))$ when there is no attack, the oracle attack scheme can force the user to pull the target arm $T-\mathcal{O}(\log(T))$ times, using a cumulative cost $|\mathcal{C}|=\mathcal{O}(\log(T))$. However, the oracle attack scheme is not practical when the true reward vector is unknown. In this paper, we will first design an effective attack scheme, which does not assume the knowledge of true reward vector and nearly matches the performance of the oracle attack scheme, to attack the UCB algorithm. We will then design a new bandit algorithm that is robust against the action-manipulation attack.

The action-manipulation attack considered here is different from reward-manipulation attacks introduced by interesting recent work~\cite{jun2018adversarial, liu2019data}, where the attacker can change the reward signal from the environment. In the setting considered in~\cite{jun2018adversarial, liu2019data}, the attacker can change the reward signal $r_t$ from the environment to an arbitrary value chosen by the attacker. Correspondingly, the cumulative attack cost in~\cite{jun2018adversarial, liu2019data} is defined to be the sum of the absolute value of the changes on the reward. Compared with the reward-manipulation attacks discussed above, the action-manipulation attack is more difficult to carry out. In particular, as the action-manipulation attack only changes the action, it can impact but does not have direct control of the reward signal, which will be a random variable drawn from a distribution depending on the action chosen by the attacker. This is in contrast to reward-manipulation attacks where an attacker can change the rewards to any value.

\section{Attack on UCB and Cost Analysis}\label{sec3}
In this section, we use UCB algorithm as an example to illustrate the effects of action-manipulation attack. We will introduce LCB attack strategy on the UCB bandit algorithm and analyze the cost.

\subsection{Attack strategy} \label{sec:as}

UCB algorithm~\cite{bubeck2012regret} is one of the most popular bandit algorithm. In the UCB algorithm, the user initially pulls each of the $K$ arms once in the first $K$ rounds. After that, the user chooses arms according to
\begin{eqnarray}
I_t=\arg \max_i \left\{ \hat{\mu}_i(t-1)+3\sigma\sqrt{\frac{\log t}{N_i(t-1)}}\right\}.\label{eq:UCB}
\end{eqnarray}
Under the action-manipulation attack, as the user does not know that $r_t$ is generated from arm $I_t^0$ instead of $I_t$, the empirical mean $\hat{\mu}_i(t)$ computed using~\eqref{eq:em} is not a proper estimate of the true mean reward of arm $i$ anymore. On the other hand, the attack is able to obtain a good estimate of $\mu_i$ by \begin{eqnarray}
\hat{\mu}_i^0(t):=N_i^0(t)^{-1}\sum_{s\in\tau_i^0(t)}r_s,\label{eq:mu0}
\end{eqnarray} where $\tau_i^0(t):= \{ s:s\le t,I_s^0=i \} $ is the set of rounds up to $t$ when the attacker changes an arm to arm $i$, and $N_i^0(t)=|\tau_i^0(t)|$ is the number of pulls of post-attack arm $i$ up to round $t$. This information gap provides a chance for attack. In this section, we assume that the target arm is not the worst arm, i.e., $\mu_K > \mu_{i_W}$. We will discuss the case where the target arm is the worst arm in Section~\ref{sec:fail}.

The proposed attack strategy works as follows. In the first $K$ rounds, the attacker does not attack. After that, at round $t$, if the user chooses a non-target arm $I_t$, the attacker changes it to arm $I_t^0$ that has the smallest lower confidence bound (LCB):

\begin{eqnarray}
I_t^0=\arg \min_i \left\{\hat{\mu}_i^0(t-1)-\textbf{CB}\left(N_i^0(t-1)\right)\right\},\label{eq:attack}
\end{eqnarray} where 
\begin{equation}
\textbf{CB}(N) = \sqrt{\frac{2\sigma^2}{N}\log\frac{\pi^2 K N^2}{3\delta}}\label{eq:CB}.
\end{equation}
Here $\delta$ is a parameter that is related to the probability statements in the analytical results presented in Section~\ref{sec:ca}. We call our scheme as LCB attack strategy. If at round $t$ the user chooses the target arm, the attacker does not attack. Thus the cumulative attack cost of our LCB attack scheme is equal to the total of times when the non-target arms are selected by the user. The algorithm is summarized in Algorithm~\ref{alg:Framwork}.

 \begin{algorithm}[htb] 
 	\caption{ LCB attack strategy on UCB algorithm} 
 	\label{alg:Framwork} 
 	\begin{algorithmic}[1] 
 		\REQUIRE ~~\\ 
 		The user's bandit algorithm, target arm $K$
 		\FOR{$t = 1, 2, \dots$}
 		\STATE The user chooses arm $I_t$ to pull according to UCB algorithm~\eqref{eq:UCB}.\
 		\IF {$I_t=K$}
 		\STATE The attacker does not attack, and $I_t^0=I_t$.
 		\ELSE 
 		\STATE The attacker attacks and changes arm $I_t$ to $I_t^0$ chosen according to~\eqref{eq:attack}. 
 		\ENDIF
 		\STATE The environment generates reward $r_t$ from arm $I_t^0$.
 		\STATE The attacker and the user receive $r_t$.
 		\ENDFOR
 	\end{algorithmic}
 \end{algorithm}

 Here, we highlight the main idea why LCB attack strategy works. As discussed in Section~\ref{sec2}, if the attacker knows which arm is the worst, the attacker can simply change the action to the worst arm when the user pulls the non-target arm. The main idea of the attack scheme is to estimate the mean of each arm, and change the non-target arm to the arm whose lower confidence bound is the smallest. Effectively, this will almost always change the non-target arm to the worst arm. More formally, for $i\neq K$, we will show that this attack strategy will ensure that $\hat{\mu}_i$ computed using~\eqref{eq:em} by the user converges to $\mu_{i_W}$. On the other hand, as the attacker does not attack when the user selects $K$, $\hat{\mu}_K$ computed by the user will still converge to the true mean $\mu_K$ with $N_K$ increasing. Because the assumption that the target arm is not the worst, which implies that $\mu_K > \mu_{i_W}$, $\hat{\mu}_i$ could be smaller than $\hat{\mu}_K$. Then the user will rarely pull the non-target arms as $\hat{\mu}_i$ is smaller than $\hat{\mu}_K$. Hence, the attack cost would also be small. The rigorous analysis of the cost will be provided in Section~\ref{sec:ca}.
 
 
\subsection{Cost analysis}\label{sec:ca}
To analyze the cost of the proposed scheme, we need to track $\hat{\mu}_i^0(t)$, the estimate obtained by the attacker using~\eqref{eq:mu0}, and $\hat{\mu}_i(t)$, the estimate obtained by the user using~\eqref{eq:em}.

The analysis of $\hat{\mu}_i^0(t)$ is relatively simple, as the attacker knows which arm is truly pulled and hence $\hat{\mu}_i^0(t)$ is the true estimate of the mean of arm $i$. Define event 
\begin{align}
    \mathcal{E}_1 := \{ \forall i, \forall t > K: |\hat{\mu}_i^0(t) - \mu_i|< \textbf{CB}(N_i^0(t)) \}.
\end{align}
Roughly speaking, event $\mathcal{E}_1$ is the event that the empirical mean computed by the attacker using~\eqref{eq:mu0} is close to the true mean. The following lemma, proved in~\cite{jun2018adversarial}, shows that the attacker can accurately estimate the average reward to each arm.
\begin{lem}\label{thm:ab}
(Lemma 1 in \cite{jun2018adversarial}) For $\delta \in (0,1)$,  $\mathbb{P}(\mathcal{E}_1)>1-\delta$.
\end{lem}

The analysis of $\hat{\mu}_i(t)$ computed by the user is more complicated. When the user pulls arm $i$, because of the action-manipulation attacks, the random rewards may be drawn from different reward distributions. Define $\tau_{i,j}(t):= \{ s:s\le t,I_s=i \ \text{and} \ I_s^0=j \}$ as the set of rounds up to $t$ when the user chooses arm $i$ and the attacker changes it to arm $j$. Lemma~\ref{thm:ab2} shows a high-probability confidence bounds of $\hat{\mu}_{i,j}(t) := N_{i,j}(t)^{-1}\sum_{s\in\tau_{i,j}(t)}r_s$, the empirical mean rewards of a part of arm $i$ whose post-attack arm is $j$, where $N_{i,j}(t) := |\tau_{i,j}(t)|$. Define event
\begin{equation}
  \begin{split}
  \mathcal{E}_2 := \left\{ \forall i, \forall j, \forall t>K: |\hat{\mu}_{i,j}(t) - \mu_j|<    \right. &\\
 \left. \sqrt{\frac{2\sigma^2}{N_{i,j}(t)}\log\frac{\pi^2 K^2 (N_{i,j}(t))^2}{3\delta}}  \right\} &.
  \end{split}
\end{equation}
\begin{lem}\label{thm:ab2}
	For  $\delta \in (0,1)$,  $\mathbb{P}(\mathcal{E}_2)>1-\delta$.
\end{lem}
\begin{proof}
	Please refer to Appendix~\ref{app:lemma2}.
\end{proof}
Although $r_s$ in \eqref{eq:em}, used to calculate $\hat{\mu}_i(t)$, may be drawn from different reward distributions, we can build a high-probability bound of $\hat{\mu}_i(t)$ with the help of Lemma~\ref{thm:ab2}. 
\begin{lem}\label{thm:equaltoab2}
Under event $\mathcal{E}_2$, for all arm $i$ and all $t>K$, we have \begin{align}
    \left| \hat{\mu}_i(t) - \frac{1}{N_i(t)}\sum_{s \in \tau_i(t)} \mu_{I_s^0}\right| < \beta(N_i(t)),
\end{align}
where
\begin{align}
\beta(N_i(t)) = \sqrt{\frac{2\sigma^2K}{N_i(t)}\log\frac{\pi^2 (N_i(t))^2}{3\delta}}.
\end{align}
\end{lem}
\begin{proof}
	Please refer to Appendix~\ref{app:lemma3}.
\end{proof}
Under events $\mathcal{E}_1$ and $\mathcal{E}_2$, we can build a connection between $\hat{u}_i(t)$ and $\mu_{i_W}$. In the proposed LCB attack strategy, the attacker explores and exploits the worst arm by a lower confidence bound method. Thus, when the user pulls a non-target arm, the attacker changes it to the worst arm at most of rounds, which means that for all $i \ne K$, $\hat{u}_i(t)$ will converge to $\mu_{i_W}$ as $N_i(t)$ increases. Lemma~\ref{thm:ab3} shows the relationship between $\hat{u}_i(t)$ and $\mu_{i_W}$.
\begin{lem}\label{thm:ab3}
Under events $\mathcal{E}_1$ and $\mathcal{E}_2$, using LCB attack strategy~\ref{alg:Framwork}, we have
\begin{align} 
    \hat{\mu}_i(t) \le u_{i_W}&+\frac{1}{N_i(t)}\sum_{j \ne i_W}\frac{8\sigma^2}{\Delta_{j,i_W}}\log \frac{\pi^2Kt^2}{3\delta}\\
    &+\sqrt{\frac{2\sigma^2K}{N_i(t)}\log\frac{\pi^2 (N_i(t))^2}{3\delta}},\forall i, t.
\end{align}
\end{lem}
\begin{proof}
	Please refer to Appendix~\ref{app:Lemma4}.
\end{proof}
Lemma~\ref{thm:ab3} shows an upper bound of the empirical mean reward of pre-attack arm $i$, for all arm $i \ne K$. Our main results is the following upper bound on the attack cost $ |\mathcal{C}|$.
\begin{theorem}\label{thm:ab4}
With probability at least $1-2\delta$, when  $T\ge\left(\frac{\pi^2K}{3\delta}\right)^{\frac{2}{5}}$, using LCB attack strategy specified in Algorithm~\ref{alg:Framwork}, the attacker can manipulate the user into pulling the target arm in at least $T - |\mathcal{C}|$ rounds, with an attack cost  
\begin{equation}
\begin{split}
    |\mathcal{C}| \le &\frac{K-1}{4\Delta_{K,i_W}^2}\left(3\sigma\sqrt{\log  T}+\sqrt{2\sigma^2K\log\frac{\pi^2 T^2}{3\delta}} \right.\\
&+\left(\left(3\sigma\sqrt{\log T}+\sqrt{2\sigma^2K\log\frac{\pi^2  T^2}{3\delta}}\right)^2 \right.\\
& \left.\left. +4\Delta_{K,i_W}\sum_{j \ne i_W}\frac{8\sigma^2}{\Delta_{j,i_W}}\log \frac{\pi^2KT^2}{3\delta}\right)^\frac{1}{2} \right) ^2.
\end{split}
\end{equation}
\end{theorem}
\begin{proof}
	Please refer to Appendix~\ref{app:thm 1}.
\end{proof}
The expression of the cost bound in Theorem~\ref{thm:ab4} is complicated. The following corollary provides a simpler bound that is more explicit and interpretable.

\begin{corollary}\label{cor:ab}
Under the same assumptions in Theorem~\ref{thm:ab4}, the total attack cost $ |\mathcal{C}|$ of Algorithm~\ref{alg:Framwork} is upper bounded by
\begin{equation}
\begin{split}
    \mathcal{O}\left(\frac{K\sigma^2\log T}{\Delta_{K,i_W}^2}\left(K+\sum_{j \ne i_W}\frac{\Delta_{K,i_W}}{\Delta_{j,i_W}}+\sqrt{K\sum_{j \neq i_W}\frac{\Delta_{K,i_W}}{\Delta_{j,i_W}}}\right)\right),
\end{split}
\end{equation}
and the total number of target arm pulls is $T- |\mathcal{C}|$.
\end{corollary}

From Corollary~\ref{cor:ab}, we can see that the attack cost scales as $\log T$. Two important constants $\frac{\sigma}{\Delta_{K,i_W}}$ and $\sum_{j \neq i_W}\frac{\Delta_{K,i_W}}{\Delta_{j,i_W}}$ have impact on the prelog factor. In Section \ref{sec4}, we provide some numerical examples to illustrate the effects of these two constants on the attack cost. 

\subsection{Attacks fail when the target arm is the worst arm} \label{sec:fail}
One weakness of our LCB attack strategy is that the attack target arm is necessarily a non-worst arm. In the LCB attack strategy, the attacker can not force the user to pull the worst arm very frequently by spending only logarithmic cost. The main reason is that, when the target arm is the worst, the average reward of each arm is larger or equal to that of the target arm. As the result, our attack scheme is not able to ensure that the target arm has a higher expected reward than the user's estimate of the rewards of other arms. In fact, the following theorem shows that  all action-manipulation attack can not manipulate the UCB algorithm into pulling the worst arm more than $T - \mathcal{O}(\log(T))$ by spending only logarithmic cost. 

\begin{theorem}\label{thm:converse}
Let $\delta < \frac{1}{2}$. Suppose the attack cost is limited by $\mathcal{O}(\log(T))$,  there is no attack that can force the UCB algorithm to pick the worst arm more than  $T-\mathcal{O}(T^\alpha)$ times with probability at least $1-\delta$, in which $\alpha < \frac{9}{64K}$.
\end{theorem}
\begin{proof}
	Please refer to Appendix~\ref{app:thm2}.
\end{proof}

This theorem shows a contrast between the case where the target arm is not the worst arm and the case where the target arm is the worst arm. If the target arm is not the worst arm, our scheme is able to force the user to pick the target arm $T-\mathcal{O}(\log(T))$ times with only logarithmic cost. On the other hand, if the target arm is the worst, Theorem~\ref{thm:converse} shows that there is no attack strategy that can force the user to pick the worst arm more than $T-\mathcal{O}(T^\alpha)$ times while incurring only logarithmic cost.

\section{Robust algorithm and regret analysis}\label{sec:defend}

The results in Section~\ref{sec3} exposes a significant security threat of the action-manipulation attacks on MABs. Under only $\mathcal{O}(\log(T))$ times of attacks carried out using the proposed LCB strategy, the UCB algorithm will almost always pull the target arm selected by the attacker. Although there are some defense algorithms~\cite{lykouris2018stochastic} and universal best arm identification schemes~\cite{8772199} for stochastic or adversarial bandit, they do not apply to action-manipulation attack setting. This motivates us to design a new bandit algorithm that is robust against action-manipulation attacks. In this section, we propose such a robust bandit algorithm and analyze its regret.

\subsection{Robust Bandit algorithm} \label{sec:ba}
In this section, we assume that a valid upper bound $A$ for the cumulative attack cost $ |\mathcal{C}|$ is known for the user, although the user do not have to know the exact cumulative attack cost $ |\mathcal{C}|$. $A$ does not need to be constant, it can scale with $T$. In other words, for a given $A$, our proposed algorithm is robust to all action-manipulation attacks with a cumulative attack cost $ |\mathcal{C}|<A$. This assumption is reasonable, as if the cost is unbounded, it will not be possible to design a robust scheme.

We first introduce some notation. Denote $\mathbf{N}(t-1):=(N_1(t-1),\dots,N_K(t-1))$ as the vector counting how many times each action has been taken by the user, and $ \boldsymbol{\hat{\mu}}(t-1)=(\hat{\mu}_1(t-1),\dots,\hat{\mu}_K(t-1)) $ as the vector of the sample means computed by the user. The proposed algorithm is a modified UCB method by taking the maximum possible mean estimate offset due to attack into consideration. We name our scheme as maximum offset UCB (MOUCB).

The proposed MOUCB works as follows. In the first $2AK$ rounds, MOUCB algorithm pulls each arm $2A$ times. After that, at round $t$, the user chooses an arm $I_t$ by a modified UCB method:
\begin{equation} \label{eq:bandit}
  \begin{split}
I_t=\arg \max_a &\left\{ \hat{\mu}_a(t-1)+\beta(N_a(t-1)) \right.\\
&\left.\hspace{5mm}+\gamma(\boldsymbol{\hat{\mu}}(t-1),\mathbf{N}(t-1)) \right\},
  \end{split}
\end{equation}
where
\begin{eqnarray}
&&\hspace{-5mm}\gamma(\boldsymbol{\hat{\mu}}(t-1),\mathbf{N}(t-1))=\frac{2A}{N_a(t-1)}\nonumber\\&&\hspace{-7mm} \max_{i,j}\left\{\hat{\mu}_i(t-1) -\hat{\mu}_j(t-1)  +\beta(N_i(t-1))+\beta(N_j(t-1)) \right\},\nonumber
\end{eqnarray}
and
\begin{equation}
  \begin{split}
\beta(N)=\sqrt{\frac{2\sigma^2K}{N}\log\frac{\pi^2 N^2}{3\delta}}.
  \end{split}
\end{equation}

The algorithm is summarized in Algorithm~\ref{alg:banditalgorithm}.

\begin{algorithm}[htb] 
\caption{Proposed MOUCB bandit algorithm} 
\label{alg:banditalgorithm} 
\begin{algorithmic}[1] 
\REQUIRE ~~\\ 
A valid upper bound $A$ for the cumulative attack cost.
\FOR{$t = 1, 2, \dots$}
\IF{$t \le 2AK$}
\STATE The user pulls the arm whose pull count is the smallest, i.e. $I_t=\arg \min_i N_i(t-1)$.\
\ELSE 
\STATE The user chooses arm $I_t$ to pull according \eqref{eq:bandit}.
\ENDIF
\IF{The attacker decides to attack}
\STATE The attacker attacks and changes $I_t$ to $I_t^0$.
\ELSE
\STATE The attacker does not attack and $I_t^0=I_t$.
\ENDIF
\STATE The environment generates reward $r_t$ from arm $I_t^0$.
\STATE The attacker and the user receive $r_t$.
\ENDFOR
\end{algorithmic}
\end{algorithm}


Compared with the original UCB algorithm in~\eqref{eq:UCB}, the main difference is the additional term $\gamma(\boldsymbol{\hat{\mu}}(t-1),\mathbf{N}(t-1))$ in~\eqref{eq:bandit}. We now highlight the main idea why our bandit algorithm works and the role of this additional term. In particular, in the standard multi-armed stochastic bandit problem, $\hat{\mu}_i(t)$ is a proper estimation of $\mu_i$, the true mean reward of arm $i$. However, under the action-manipulation attacks, as the user does not know which arm is used to generate $r_t$, $\hat{\mu}_i(t)$ is not a proper estimate of the true mean reward anymore. However, we can try to find a good bound of the true mean reward. If we know $\Delta_{i_O,i_W}$, the reward difference between the optimal arm and the worst arm, we can describe the maximum offset of the mean rewards caused by the attack. In particular, we have
\begin{equation}
\mu_i - \frac{A}{N_i(t)}\Delta_{i_O,i_W} \le  \frac{1}{N_i(t)}\sum_{s \in \tau_i(t)} \mu_{I_s^0} \le \mu_i + \frac{A}{N_i(t)}\Delta_{i_O,i_W},
\end{equation}
which implies
\begin{equation}
\mu_i \le  \frac{A}{N_i(t)}\Delta_{i_O,i_W} + \frac{1}{N_i(t)}\sum_{s \in \tau_i(t)} \mu_{I_s^0}.\label{eq:mou}
\end{equation}

In~\eqref{eq:mou}, the first term in the right hand side is the maximum offset that an attacker can introduce regardless of the attack strategy. The second term in the right hand side is related to the mean estimated by the user. In particular, under event $\mathcal{E}_2$, as shown in Lemma~\ref{thm:equaltoab2}, we have
\begin{eqnarray}
     \frac{1}{N_i(t)}\sum_{s \in \tau_i(t)} \mu_{I_s^0} < \hat{\mu}_i(t) + \beta(N_i(t)).
\end{eqnarray}

Hence, regardless the attack strategy, we have a upper confidence bound on $\mu_i$:
\begin{equation}
\mu_i \le \hat{\mu}_i(t) + \frac{A}{N_i(t)}\Delta_{i_O,i_W} + \beta(N_i(t)).\label{eq:Delta}
\end{equation}
In our case, however, $\Delta_{i_O,i_W}$ is also unknown. In our algorithm, we obtain a high-probability bound on $\Delta_{i_O,i_W}$: \begin{eqnarray}\Delta_{i_O,i_W}\leq 2\max_{i,j}\left\{\hat{\mu}_i-\hat{\mu}_j+\beta\left(N_i(t)\right)+\beta\left(N_j(t)\right) \right\},\label{eq:Deltabound}\end{eqnarray}
which will proved in Lemma~\ref{thm:banditdelta} below. Now, the second term of~\eqref{eq:Delta} becomes $\gamma(\boldsymbol{\hat{\mu}}(t-1),\mathbf{N}(t-1))$ if we replace $\Delta_{i_O,i_W}$ with the bound~\eqref{eq:Deltabound}, and we obtain our final algorithm.



\subsection{Regret analysis} \label{sec:RA}
 Lemma~\ref{thm:banditdelta} shows a boundary of $\Delta_{i_O,i_W}$ the maximum reward difference between any two arms, under event $\mathcal{E}_2$.
\begin{lem}\label{thm:banditdelta}
	For  $\delta \le \frac{1}{3}$, $t>2AK$ and under event $\mathcal{E}_2$, MOUCB algorithm have
\begin{equation}
  \begin{split}
	\Delta_{i_O,i_W} &\le 2\max_{i,j}\left\{\hat{\mu}_i-\hat{\mu}_j+\beta\left(N_i(t)\right)+\beta\left(N_j(t)\right) \right\} \\
	&\le 2\Delta_{i_O,i_W}+8\sqrt{\frac{\sigma^2K}{A}\log\frac{4\pi^2 A^2}{3\delta}}.
	\end{split}
\end{equation}
\end{lem}
\begin{proof}
	Please refer to Appendix~\ref{app:lemma5}.
\end{proof}

Using Lemma~\ref{thm:banditdelta}, we now bound the regret of Algorithm~\ref{alg:banditalgorithm}.

\begin{theorem}\label{thm:proposed}
Let $A$ be an upper bound on the total attack cost $ |\mathcal{C}|$. For $\delta \le \frac{1}{3}$ and $T \ge 2AK$, MOUCB algorithm has pseudo-regret $\bar{R}(T)$ 
\begin{equation}
  \begin{split}
\bar{R}(T) \le \sum_{a \not= i_O} \max&\left\{\frac{8\sigma^2K}{\Delta_{i_O,a}}\log\frac{\pi^2T^2}{3\delta}, 
A \left( \Delta_{i_O,a} \right.\right.\\
&\left.\left.+ 2\Delta_{i_O,i_W} + 8\sqrt{\frac{\sigma^2K}{A}\log\frac{4\pi^2 A^2}{3\delta}}\right)\right\},
\end{split}
\end{equation}
with probability at least $1-\delta$.
\end{theorem}
\begin{proof}
	Please refer to Appendix~\ref{app:thm3}.
\end{proof}

Theorem~\ref{thm:proposed} reveals that our bandit algorithm is robust to the action-manipulation attacks. If the total attack cost is bounded by $\mathcal{O}(\log T)$, the pseudo-regret of MOUCB bandit algorithm is still bounded by $\mathcal{O}(\log T)$. This is in contrast with UCB, for which we have shown that the pseudo-regret is $\mathcal{O}(T)$ with attack cost $\mathcal{O}(\log T)$ in Section~\ref{sec3}. If the total attack cost is up to $\Omega(\log T)$, the pseudo-regret of MOUCB bandit algorithm is bounded by $\mathcal{O}(A)$, which is linear in $A$.

\section{Numerical results}\label{sec4}
In this section, we provide numerical examples to illustrate the analytical results obtained. In our simulation, the bandit has 10 arms. The rewards distribution of arm $i$ is $\mathcal{N}(\mu_i, \sigma)$. The attacker's target arm is $K$. We let $\delta=0.05$. We then run the experiment for multiple trials and in each trial we run $T=10^7$ rounds.

\subsection{LCB attack strategy}

We first illustrate the impact of the proposed LCB attack strategy on UCB algorithm.

\begin{figure}[ht]
\centering
\includegraphics[width=0.8\linewidth]{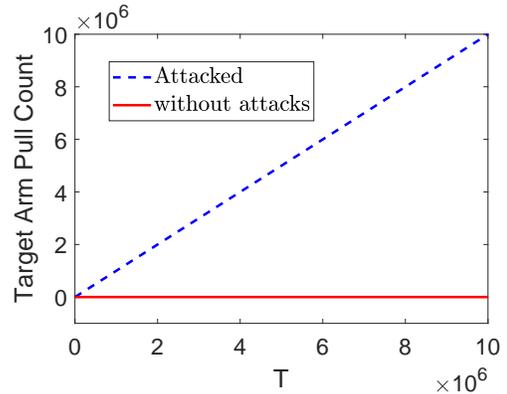}
\caption{Number of rounds the target arm was pulled}
\label{fig:num}
\end{figure}

In Figure \ref{fig:num}, we fix $\sigma = 0.1$ and $\Delta_{K,i_W}=0.1$ and compare the number of rounds at which the target arm is pulled with and without attack. In this experiment, the mean rewards of all arms are 1.0, 0.9, 0.8, 0.7, 0.6, 0.5, 0.4, 0.3, 0.1, and 0.2 respectively. Arm $K$ is not the worst arm, but its average reward is lower than most arms. The results are averaged over 20 trials. The attacker successfully manipulates the user into pulling the target arm very frequently.

\begin{figure}[ht]
\centering
\includegraphics[width=0.8\linewidth]{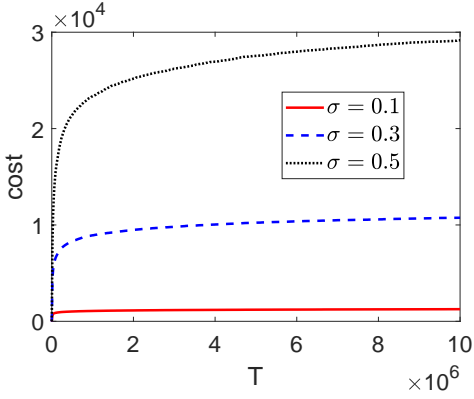}
\caption{Attack cost vs $\frac{\sigma}{\Delta_{K,i_W}}$}
\label{fig:sigma}
\end{figure}

In Figure \ref{fig:sigma}, in order to study how $\frac{\sigma}{\Delta_{K,i_W}}$ affects the attack cost, we fix $\Delta_{K,i_W}=0.1$ and set $\sigma$ as 0.1, 0.3 and 0.5 respectively. The mean rewards of all arms are same as above. From the figure, we can see that as $\frac{\sigma}{\Delta_{K,i_W}}$ increases, the attack cost increases. In addition, as predicted in our analysis, the attack cost increases with $T$, the total number of rounds, in a logarithmic order.

\begin{figure}[ht]
\centering
\includegraphics[width=0.8\linewidth]{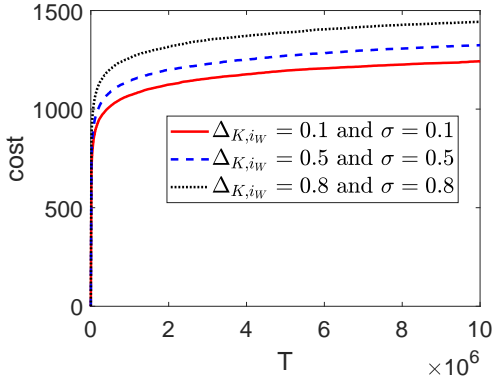}
\caption{Attack cost vs $\sum_{j \neq i_W}\frac{\Delta_{K,i_W}}{\Delta_{j,i_W}}$}
\label{fig:Delta}
\end{figure}

Figure \ref{fig:Delta} illustrates how $\sum_{j \neq i_W}\frac{\Delta_{K,i_W}}{\Delta_{j,i_W}}$ affects the attack cost. In this experiment, we fix $\frac{\sigma}{\Delta_{K,i_W}}=1$ and set $\Delta_{K,i_W}$ as 0.2, 0.6 and 0.9 respectively. The mean rewards of all arms are the same as above. The figure illustrates that, as $\sum_{j \neq i_W}\frac{\Delta_{K,i_W}}{\Delta_{j,i_W}}$ increases, the attack cost also increases. This is consistent with our analysis in Corollary~\ref{cor:ab}.

\subsection{MOUCB bandit algorithm}

We now illustrate the effectiveness of MOUCB bandit algorithm.

In this experiment, we use the similar setting as in the simulation of the LCB attack scheme. The mean rewards of all arms are set to be 1.0, 0.8, 0.9, 0.5, 0.2, 0.3, 0.1, 0.4, 0.7, and 0.6 respectively. The total attack cost $ |\mathcal{C}|$ is limited by 2000. A given valid upper bound for total attack cost is $A = 3000$. The results are averaged over 20 trials.

\begin{figure}[htbp]
\centering
\includegraphics[width=0.8\linewidth]{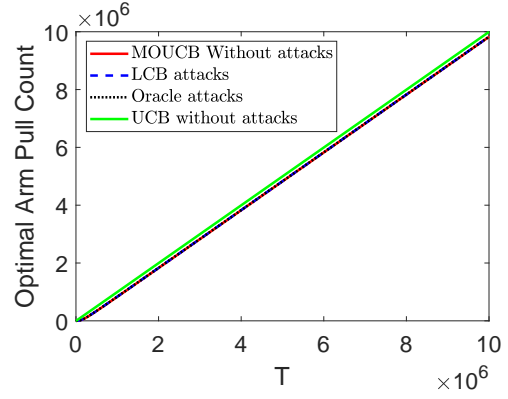}
\caption{Comparison of number of rounds the optimal arm was pulled }
\label{fig:numour}
\end{figure}
In Figure \ref{fig:numour}, we simulate MOUCB algorithm with two different attacks, and compare the numbers of rounds when the optimal arm is pulled under these attacks. The first attack is the LCB attack discussed in Section~\ref{sec3}. The second attack is the oracle attack, in which the attacker knows the true mean reward of arms and implements the oracle attacks that change any non-target arm to a worst arm (see the discussion in Section~\ref{sec2}). For comparison purposes, we also add the curve for MOUCB under no attack, and the curve for UCB under no attack. The results show that, even under the oracle attack, the proposed MOUCB bandit algorithm achieves almost the same performance as the UCB without attack. 

\begin{figure}[htbp]
\centering
\includegraphics[width=0.8\linewidth]{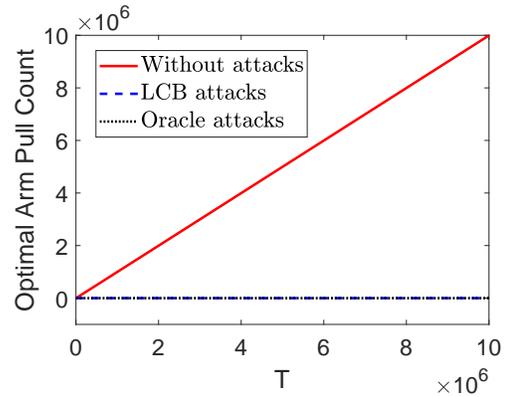}
\caption{Number of rounds the optimal arm was pulled using UCB algorithm}
\label{fig:numucb}
\end{figure}
To further compare the performance of UCB and MOUCB, in Figure \ref{fig:numucb}, we illustrate the performance of UCB algorithm for the three scenarios discussed above: under LCB attack, under oracle attack and under no attack. The results show that both LCB and oracle attacks can successfully manipulates the UCB algorithm into pulling a non-optimal arm very frequently, as the curves for the LCB attack and oracle attack are far away from the curve for no attack. This is in sharp contrast with the situation for MOUCB algorithm shown in Figure~\ref{fig:numour}, where the all curves are almost identical. 

\begin{figure}[htbp]
\centering
\includegraphics[width=0.8\linewidth]{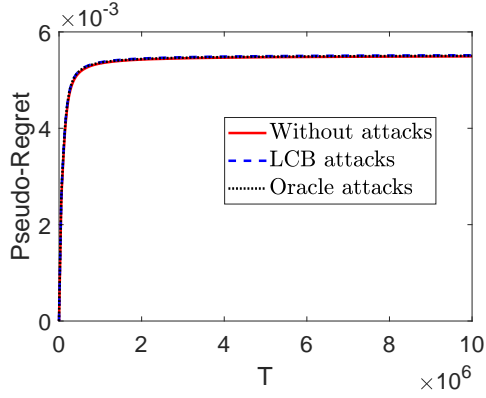}
\caption{Pseudo-regret of MOUCB algorithm}
\label{fig:regour}
\end{figure}

\begin{figure}[htbp]
\centering
\includegraphics[width=0.8\linewidth]{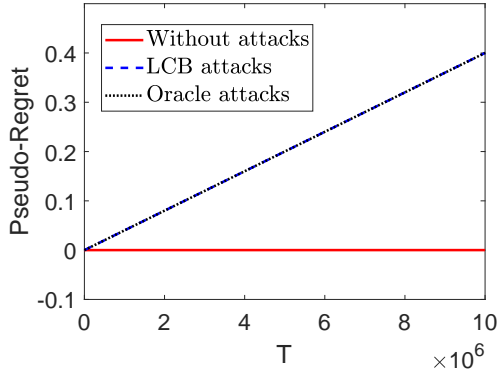}
\caption{Pseudo-regret of UCB algorithm}
\label{fig:regucb}
\end{figure}

 Figure \ref{fig:regour} and Figure \ref{fig:regucb} illustrate the pseudo-regret of MOUCB bandit algorithm and UCB bandit algorithm respectively. In Figure \ref{fig:regour}, as predicted in our analysis, MOUCB algorithm archives logarithmic pseudo-regrets under both LCB attacks and the oracle attacks. Furthermore, the curves under both attacks are very close to that of the case without attacks. However, as shown in Figure \ref{fig:regucb}, the pseudo-regret of UCB grows linearly under both attacks, while grows logarithmically under no attack. The figures again show that UCB is vulnerable to action-manipulation attacks while the proposed MOUCB is robust to the attacks (even for oracle attacks). 
 

\section{Conclusion}\label{sec5}
In this paper, we have introduced a new class of attacks on stochastic bandits: action-manipulation attacks. We have analyzed the attack against on the UCB algorithm and proved that the proposed LCB attack scheme can force the user to almost always pull a non-worst arm with only logarithm effort. 
To defend against this type of attacks, we have further designed a new bandit algorithm MOUCB that is robust to action-manipulation attacks. We have analyzed the regret of MOUCB under any attack with bounded cost, and have showed that the proposed algorithm is robust to the action-manipulation attacks. 
\appendices

\section{Proof of Lemma~\ref{thm:ab2}}\label{app:lemma2}
The proof is similar with the proof of Lemma~\ref{thm:ab} that was proved in~\cite{jun2018adversarial}. Let $\{ X_j \}_{j=1}^\infty$ be a sequence of $i.i.d$ $\sigma^2$-sub-Gaussian random variables with mean $\mu$. Let $\hat{\mu}^0(t)=\frac{1}{N(t)}\sum_{j=1}^{N(t)}X_j$. By Hoeffding's inequality.
\begin{equation}
\mathbb{P}(|\hat{\mu}^0(t) - \mu|\ge \eta) \le 2 \exp \left(-\frac{N(t)\eta^2}{2\sigma^2} \right).
\end{equation}
In order to ensure that $\mathcal{E}_2$ holds for all arm $i$, all arm $j$ and all pull counts $N = N_{i,j}(t)$, we set $\delta_{i,j,N}:=\frac{6\delta}{\pi^2K^2N^2}$. We have
\begin{equation}
\begin{split}
&\mathbb{P}\left(\exists i, \exists j, \exists N:|\hat{\mu}_{i,j}(t) - \mu_j| \ge  \sqrt{\frac{2\sigma^2}{N}\log\frac{\pi^2 K^2 N^2}{3\delta}} \right)\\
& = \sum_{i=1}^{K} \sum_{j=1}^K \sum_{N=1}^\infty \delta_{i,N} = \delta. 
\end{split}
\end{equation}

\section{Proof of Lemma~\ref{thm:equaltoab2}}\label{app:lemma3}

According to event $\mathcal{E}_2$, we have

\begin{equation}
\begin{split}
&\left| \hat{\mu}_i(t) - \frac{1}{N_i(t)}\sum_{s \in \tau_i(t)} \mu_{I_s^0}\right| \\
= & \left| \sum_{j=1}^K\frac{N_{i,j}(t)}{N_i(t)}\left(\hat{\mu}_{i,j}(t)-\mu_j \right)\right|\\
\le & \sum_{j=1}^K \frac{N_{i,j}(t)}{N_i(t)}\left|\hat{\mu}_{i,j}(t)-\mu_j \right| \\
< &  \frac{1}{N_i(t)} \sum_{j=1}^K \sqrt{2\sigma^2 N_{i,j}(t)\log\frac{\pi^2 K^2 (N_{i,j}(t))^2}{3\delta}}.
\end{split}
\end{equation}

Define a function $\displaystyle f(N) = \sqrt{2\sigma^2 N\log\frac{\pi^2 K^2 N^2}{3\delta}}: \mathbb {R} \to \mathbb {R}$, and we have
\begin{equation}
\begin{split}
\displaystyle f''(N) =& \frac{\partial^2}{\partial N^2}\sqrt{2\sigma^2 N\log\frac{\pi^2 K^2 N^2}{3\delta}} \\
=&-\frac{\left(2\sigma^2 \log\frac{\pi^2 K^2 N^2}{3\delta}\right)^2+16\sigma^4}{4\left(2\sigma^2 N\log\frac{\pi^2 K^2 N^2}{3\delta}\right)^{\frac{3}{2}}}\\
<&0,
\end{split}
\end{equation}
when $N \ge 1$.

Hence $\displaystyle f$ is strictly concave when $N \ge 1$, and we have
\begin{equation} 
\begin{split}
\displaystyle \sum_{j=1}^K f(N_{i,j}(t)) < K f\left(\frac{1}{K}\sum_{j=1}^K N_{i,j}(t)\right) = K f\left(\frac{N_i(t)}{K}\right).
\end{split}
\end{equation}

Thus,
\begin{equation} \label{eq:l3}
\begin{split}
&\left| \hat{\mu}_i(t) - \frac{1}{N_i(t)}\sum_{s \in \tau_i(t)} \mu_{I_s^0}\right| \\
<&\frac{1}{N_i(t)} K \sqrt{2\sigma^2 \frac{N_i(t)}{K}\log\frac{\pi^2 K^2 \left(\frac{N_i(t)}{K}\right)^2}{3\delta}}\\
=&\sqrt{\frac{2\sigma^2K}{N_i(t)}\log\frac{\pi^2 (N_i(t))^2}{3\delta}}.
\end{split}
\end{equation}

\section{Proof of Lemma~\ref{thm:ab3}}\label{app:Lemma4}
The LCB attack scheme uses lower confidence bound to exploit the worst arm, so we need to prove that the attacker's pull counts of all non-worst arms should be limited at round $t$. 

Consider the case that in round $t+1$, the user chooses a non-target arm $I_{t+1} = i \ne K$ and the attacker changes it to a non-worst arm $I_{t+1}^0 = j \ne i_W$. On one hand, under event $\mathcal{E}_1$, we have
\begin{equation} 
\begin{split}
&\hat{\mu}_{i_W}^0(t) - \mu_{i_W}< \textbf{CB}(N_{i_W}^0(t)),\\
\text{and} \ &\hat{\mu}_j^0(t) - \mu_j> -\textbf{CB}(N_j^0(t)).
\end{split} \label{eq:fromevent1}
\end{equation}

On the other hand, according to the attack scheme, it must be the case that
\begin{equation} 
\begin{split}
\hat{\mu}_{i_W}^0(t) - \textbf{CB}(N_{i_W}^0(t)) > \hat{\mu}_j^0(t) -\textbf{CB}(N_j^0(t)),
\end{split}
\end{equation}
which is equivalent to
\begin{equation} 
\begin{split}
\textbf{CB}(N_j^0(t)) > \hat{\mu}_j^0(t) - (\hat{\mu}_{i_W}^0(t) - \textbf{CB}(N_{i_W}^0(t)) ).\label{eq:cb}
\end{split}
\end{equation}
Combining~\eqref{eq:cb} with~\eqref{eq:fromevent1}, we have
\begin{equation} 
\begin{split}
&\textbf{CB}(N_j^0(t)) > \mu_j - \textbf{CB}(N_j^0(t)) - \mu_{i_W} \\
\text{and} \ &\textbf{CB}(N_j^0(t)) > \frac{\Delta_{j,i_W}}{2}.
\end{split}
\end{equation}
Using on the fact that $N_j^0(t) \le t$ and $N_{i,j}(t) \le N_j^0(t)$, we have 
\begin{equation} 
\begin{split}
\frac{\Delta_{j,i_W}}{2} < &\textbf{CB}(N_j^0(t)) \\
=&\sqrt{\frac{2\sigma^2}{N_j^0(t)}\log\frac{\pi^2 K (N_j^0(t))^2}{3\delta}}\\
\le&\sqrt{\frac{2\sigma^2}{N_j^0(t)}\log\frac{\pi^2 K t^2}{3\delta}}\\
\le&\sqrt{\frac{2\sigma^2}{N_{i,j}(t)}\log\frac{\pi^2 K t^2}{3\delta}},
\end{split}
\end{equation}
which is equivalent to
\begin{equation} 
\begin{split}
N_{i,j}(t)<\frac{8\sigma^2}{\Delta_{j,i_W}^2}\log \frac{\pi^2Kt^2}{3\delta}.
\end{split}
\end{equation}

Hence, under event $\mathcal{E}_2$, we have
\begin{equation} 
\begin{split}
 \hat{\mu}_i(t) < &\frac{1}{N_i(t)}\sum_{s \in \tau_i(t)} \mu_{I_s^0} + \sqrt{\frac{2\sigma^2K}{N_i(t)}\log\frac{\pi^2 (N_i(t))^2}{3\delta}}\\
 = &\frac{1}{N_i(t)}\sum_j\sum_{s \in \tau_{i,j}(t)} \mu_{I_s^0} + \sqrt{\frac{2\sigma^2K}{N_i(t)}\log\frac{\pi^2 (N_i(t))^2}{3\delta}}\\
 = &\frac{1}{N_i(t)}\sum_j N_{i,j}(t) \mu_j + \sqrt{\frac{2\sigma^2K}{N_i(t)}\log\frac{\pi^2 (N_i(t))^2}{3\delta}}\\
 = &\sum_j \frac{N_{i,j}(t)}{N_i(t)} (\Delta_{j,i_W}+\mu_{i_W}) + \sqrt{\frac{2\sigma^2K}{N_i(t)}\log\frac{\pi^2 (N_i(t))^2}{3\delta}}\\
 <&\mu_{i_W}+\sqrt{\frac{2\sigma^2K}{N_i(t)}\log\frac{\pi^2 (N_i(t))^2}{3\delta}}\\
    &+\frac{1}{N_i(t)}\sum_{j \ne i_W}\frac{8\sigma^2}{\Delta_{j,i_W}}\log \frac{\pi^2Kt^2}{3\delta}.
\end{split}
\end{equation}
The lemma is proved.
\section{Proof of Theorem~\ref{thm:ab4}} \label{app:thm 1}
By inferring from Lemma~\ref{thm:ab}, we have that with probability $1-\frac{\delta}{K}$, $\forall t > K: |\hat{\mu}_K^0(t) - \mu_K|< \textbf{CB}(N_K^0(t))$.

Because the LCB attack scheme does not attack the target arm, we can also conclude that with probability $1-\frac{\delta}{K}$, $\forall t > K: |\hat{\mu}_K(t) - \mu_K|< \textbf{CB}(N_K(t))$.

The user relies on the UCB algorithm to choose arms. If at round $t$, the user chooses an arm $I_t = i \ne K$, which is not the target arm, we have
\begin{equation} 
\begin{split}
\hat{\mu}_i(t-1)+3\sigma\sqrt{\frac{\log t}{N_i(t-1)}}>\hat{\mu}_K(t-1)+3\sigma\sqrt{\frac{\log t}{N_K(t-1)}},
\end{split}
\end{equation}
which is equivalent to
\begin{equation} 
\begin{split}
3\sigma\sqrt{\frac{\log t}{N_i(t-1)}}>- \hat{\mu}_i(t-1)+\hat{\mu}_K(t-1)+3\sigma\sqrt{\frac{\log t}{N_K(t-1)}}.
\end{split}
\end{equation}
Under event $\mathcal{E}_1$, we have
\begin{equation} 
\begin{split}
\hat{\mu}_K(t) > \mu_K  - \textbf{CB}(N_K(t)).
\end{split}
\end{equation}
Under event $\mathcal{E}_1\cap\mathcal{E}_2 $, according to Lemma~\ref{thm:ab3}, we have
\begin{equation} 
\begin{split}
\hat{\mu}_i(t) \le \mu_{i_W}&+\sqrt{\frac{2\sigma^2K}{N_i(t)}\log\frac{\pi^2 (N_i(t))^2}{3\delta}}\\
    &+\frac{1}{N_i(t)}\sum_{j \ne i_W}\frac{8\sigma^2}{\Delta_{j,i_W}}\log \frac{\pi^2Kt^2}{3\delta}.
\end{split}
\end{equation}
Combing the inequalities above,
\begin{equation} 
\begin{split}
3\sigma\sqrt{\frac{\log t}{N_i(t-1)}}&>-\mu_{i_W}-\sqrt{\frac{2\sigma^2K}{N_i(t-1)}\log\frac{\pi^2 (N_i(t-1))^2}{3\delta}} \\
&-\frac{1}{N_i(t-1)}\sum_{j \ne i_W}\frac{8\sigma^2}{\Delta_{j,i_W}}\log \frac{\pi^2K(t-1)^2}{3\delta}  \\
&+\mu_K - \textbf{CB}(N_K(t-1))+3\sigma\sqrt{\frac{\log t}{N_K(t-1)}}.
\end{split}
\end{equation}
When $t \ge (\frac{\pi^2K}{3\delta})^\frac{2}{5}$,
\begin{equation} 
\begin{split}
3\sigma\sqrt{\frac{\log t}{N_K(t-1)}} &\ge \sqrt{4\sigma^2\frac{\log t}{N_K(t-1)}+5\sigma^2\frac{\log (\frac{\pi^2K}{3\delta})^\frac{2}{5}}{N_K(t-1)}} \\
&\ge \sqrt{2\sigma^2\frac{\log \frac{\pi^2Kt^2}{3\delta}}{N_K(t-1)}} \\
&\ge \sqrt{2\sigma^2\frac{\log \frac{\pi^2K(N_K(t-1))^2}{3\delta}}{N_K(t-1)}} \\
&= \textbf{CB}(N_K(t-1)).
\end{split}
\end{equation}
Now the inequality only depends on $N_i(t-1)$ and some constants:
\begin{equation} 
\begin{split}
3\sigma\sqrt{\frac{\log t}{N_i(t-1)}}>&\Delta_{K,i_W}-\sqrt{\frac{2\sigma^2K}{N_i(t-1)}\log\frac{\pi^2 (N_i(t-1))^2}{3\delta}} \\
&-\frac{1}{N_i(t-1)}\sum_{j \ne i_W}\frac{8\sigma^2}{\Delta_{j,i_W}}\log \frac{\pi^2K(t-1)^2}{3\delta}  \\
>&\Delta_{K,i_W}-\sqrt{\frac{2\sigma^2K}{N_i(t-1)}\log\frac{\pi^2 t^2}{3\delta}} \\
&-\frac{1}{N_i(t-1)}\sum_{j \ne i_W}\frac{8\sigma^2}{\Delta_{j,i_W}}\log \frac{\pi^2Kt^2}{3\delta}.
\end{split}
\end{equation}
By solving the inequality above, we have:
\begin{equation} 
\begin{split}
N_i(t-1)<&\frac{1}{4\Delta_{K,i_W}^2}\left(3\sigma\sqrt{\log t}+\sqrt{2\sigma^2K\log\frac{\pi^2 t^2}{3\delta}}\right.\\
&+\left(\left(3\sigma\sqrt{\log t}+\sqrt{2\sigma^2K\log\frac{\pi^2 t^2}{3\delta}}\right)^2\right.\\
&\left.\left. +4\Delta_{K,i_W}\sum_{j \ne i_W}\frac{8\sigma^2}{\Delta_{j,i_W}}\log \frac{\pi^2Kt^2}{3\delta}\right)^\frac{1}{2}\right)^2. \label{eq:th1eq}
\end{split}
\end{equation}
Since event $\mathcal{E}_1\cap\mathcal{E}_2 $ occurs with probability at least $1-2\delta$, we have that \eqref{eq:th1eq} holds with probability at least $1- 2\delta$. Theorem~\ref{thm:ab4} follows immediately from the definition of the attack cost and \eqref{eq:th1eq}.
\section{Proof of Theorem~\ref{thm:converse}}\label{app:thm2}
Because the target arm is the worst arm, the mean rewards of all arms are larger than or equal to that of the target arm. Thus, for any attack scheme, we have
\begin{equation} 
\frac{1}{N_i(t)}\sum_{s \in \tau_i(t)} \mu_{I_s^0} \ge \mu_K.
\end{equation}

If the user pulls arm $K$ at round $t$, according to UCB algorithm, we have for any arm $i \ne K$,
\begin{equation} 
\begin{split}
\hat{\mu}_i(t-1)+3\sigma\sqrt{\frac{\log t}{N_i(t-1)}}<\hat{\mu}_K(t-1)+3\sigma\sqrt{\frac{\log t}{N_K(t-1)}}.
\end{split}
\end{equation}

Under event $\mathcal{E}_2$, we have Lemma~\ref{thm:equaltoab2} and \eqref{eq:l3} holds for all arm $i$, which implies 
\begin{equation} 
\begin{split}
     \hat{\mu}_i(t-1) > & \frac{1}{N_i(t-1)}\sum_{s \in \tau_i(t-1)} \mu_{I_s^0} - \\ &\sqrt{\frac{2\sigma^2K}{N_i(t-1)}\log\frac{\pi^2 (N_i(t-1))^2}{3\delta}}, 
\end{split}
\end{equation}
and
\begin{equation} 
\begin{split}
\hat{\mu}_K(t-1) < & \frac{1}{N_K(t-1)}\sum_{s \in \tau_K(t-1)} \mu_{I_s^0} + \\ &\sqrt{\frac{2\sigma^2K}{N_K(t-1)}\log\frac{\pi^2 (N_K(t-1))^2}{3\delta}}.
\end{split}
\end{equation}

Noted that for $\delta > \frac{1}{2}$, $\displaystyle h(N) = 2\sigma^2 \frac{N}{K}\log\frac{\pi^2N^2}{3\delta}: \mathbb {R} \to \mathbb {R}$ is monotonically decreasing in $N \ge 1$.

If $N_i(t-1)<\frac{1}{16} N_K(t-1)$ and $N_i(t-1)<\frac{\sqrt{3\delta}}{\pi}t^{\frac{9}{64K}}$ hold for any arm $i$, we have
\begin{equation} 
\begin{split}
3\sigma\sqrt{\frac{\log t}{N_K(t-1)}}<\frac{3}{4}\sigma\sqrt{\frac{\log t}{N_i(t-1)}},
\end{split}
\end{equation}
and
\begin{equation} 
\begin{split}
&\sqrt{\frac{2\sigma^2K}{N_K(t-1)}\log\frac{\pi^2 (N_K(t-1))^2}{3\delta}}\\
<&\sqrt{\frac{2\sigma^2K}{N_i(t-1)}\log\frac{\pi^2 (N_i(t-1))^2}{3\delta}}\\
<&\frac{3}{4}\sigma\sqrt{\frac{\log t}{N_i(t-1)}}.
\end{split}
\end{equation}

Combining the inequalities above, we find 
\begin{equation} 
\begin{split}
\frac{3}{4}\sigma\sqrt{\frac{\log t}{N_i(t-1)}} < \frac{1}{N_K(t-1)}\sum_{s \in \tau_K(t-1)}  \mu_{I_s^0} - \mu_K.
\end{split}
\end{equation}

Since the attack cost is limited in $\mathcal{O}(\log t)$, 
\textbf{\begin{equation} 
\begin{split}
\frac{1}{N_K(t-1)}\sum_{s \in \tau_K(t-1)}  \mu_{I_s^0} - \mu_K = \frac{\mathcal{O}(\log t)}{N_K(t-1)},
\end{split}
\end{equation}}
so
\textbf{\begin{equation} 
\begin{split}
N_i(t-1) = \Omega(\sigma(N_K(t-1))^2).
\end{split}
\end{equation}}

In summary, as long as the event $\mathcal{E}_2$ holds, at least one of the three following equations must be true:
\textbf{\begin{equation} 
\begin{split}
&N_i(t-1) = \Omega(\sigma(N_K(t-1))^2), \\
&N_i(t-1) \ge \frac{1}{16} N_K(t-1), \\
&N_i(t-1) \ge \frac{\sqrt{3\delta}}{\pi}t^{\frac{9}{64K}}.
\end{split}
\end{equation}}
In addition, any one of the three equations shows that the user pulls the non-target arm more than $\mathcal{O}(t^\alpha)$ times, in which $\alpha < \frac{9}{64K}$. Since event $\mathcal{E}_2$ holds with probability at least $1-\delta$, the conclusion in the Theorem holds with probability at least $1-\delta$.

\section{Proof of Lemma~\ref{thm:banditdelta}} \label{app:lemma5}

Note that for $\delta\le \frac{1}{3}$, $ \beta(N)= \sqrt{\frac{2\sigma^2K}{N}\log\frac{\pi^2N^2}{3\delta}}$ is monotonically decreasing in $N$, as
\begin{equation} 
\begin{split}
 \frac{\partial}{\partial N}\beta^2(N) =& \frac{2\sigma^2K}{N^2}\left(2-\log\frac{\pi^2N^2}{3\delta}\right)\\
\le&\frac{2\sigma^2K}{N^2}\left(2-\log\frac{\pi^2}{3\delta}\right) < 0.
\end{split}
\end{equation}

We first prove the first inequality in Lemma~\ref{thm:banditdelta}. Consider the optimal arm $i_O$ and the worst arm $i_W$. Define $C_i:= |\{ t: t\le T,I_t^0\not= I_t=i \}|$. In the action-manipulation setting, when $t>2AK$, MOUCB algorithm has 
\begin{equation} 
\begin{split}
\frac{1}{N_{i_O}(t)}\sum_{s \in \tau_{i_O}(t)} \mu_{I_s^0} &\ge \frac{N_{i_O}(t) - C_{i_O}}{N_{i_O}(t)}\mu_{i_O}+\frac{ C_{i_O}}{N_{i_O}(t)}\mu_{i_W} \\
&= \mu_{i_O}-\Delta_{i_O,i_W}\frac{ C_{i_O}}{N_{i_O}(t)} \\
&\ge \mu_{i_O}-\Delta_{i_O,i_W}\frac{ C_{i_O}}{2A},\label{eq:nio}
\end{split}
\end{equation}
and 
\begin{equation} 
\begin{split}
\frac{1}{N_{i_W}(t)}\sum_{s \in \tau_{i_W}(t)} \mu_{I_s^0} &\le \frac{N_{i_W}(t) - C_{i_W}}{N_{i_W}(t)}\mu_{i_W}+\frac{ C_{i_W}}{N_{i_W}(t)}\mu_{i_O} \\
&= \mu_{i_W}+\Delta_{i_O,i_W}\frac{ C_{i_W}}{N_{i_W}(t)} \\
&\le \mu_{i_W}+\Delta_{i_O,i_W}\frac{ C_{i_W}}{2A}.\label{eq:niw}
\end{split}
\end{equation}

Combining~\eqref{eq:nio} and~\eqref{eq:niw}, we have
\begin{equation} 
\begin{split}
&\frac{1}{N_{i_O}(t)}\sum_{s \in \tau_{i_O}(t)} \mu_{I_s^0}-\frac{1}{N_{i_W}(t)}\sum_{s \in \tau_{i_W}(t)} \mu_{I_s^0} \\
\ge & \mu_{i_O}-\mu_{i_W}-\Delta_{i_O,i_W}\frac{ C_{i_W}}{2A} -\Delta_{i_O,i_W}\frac{ C_{i_O}}{2A} \\
\ge &\mu_{i_O}-\mu_{i_W}-\Delta_{i_O,i_W}\frac{A}{2A} \\
= & \frac{\Delta_{i_O,i_W}}{2}.
\end{split}
\end{equation}

From \eqref{eq:l3}, we could find
\begin{equation} 
\begin{split}
&\frac{1}{N_{i_O}(t)}\sum_{s \in \tau_{i_O}(t)} \mu_{I_s^0}-\frac{1}{N_{i_W}(t)}\sum_{s \in \tau_{i_W}(t)} \mu_{I_s^0} \\
\le & \hat{\mu}_{i_O}(t)+\beta(N_{i_O}(t))-(\hat{\mu}_{i_W}(t)-\beta(N_{i_W}(t))) \\
\le & \max_{i,j}\left\{\hat{\mu}_i(t)+\beta(N_i(t))-\left(\hat{\mu}_j(t)-\beta\left(N_j(t)\right)\right)\right\}.
\end{split}
\end{equation}

We now prove the second inequality in Lemma~\ref{thm:banditdelta}:
\begin{equation} 
\begin{split}
 &\max_{i,j}\left\{\hat{\mu}_i(t)+\beta(N_i(t))-\left(\hat{\mu}_j(t)-\beta(N_j(t))\right)\right\}\\
 \le & \max_{i,j}\left\{\frac{1}{N_i(t)}\sum_{s \in \tau_i(t)} \mu_{I_s^0}+2\beta(N_i(t))\right. \\
& \left. -\left(\frac{1}{N_i(t)}\sum_{s \in \tau_i(t)}\mu_{I_s^0}-2\beta(N_j(t))\right)\right\} \\
 \le &\Delta_{i_O,i_W}+\max_{i,j}\left\{2\beta(N_i(t))+2\beta(N_j(t))\right\}.
\end{split}
\end{equation}

Recall that for $\delta\le \frac{1}{3}$, $ \beta(N)= \sqrt{\frac{2\sigma^2K}{N}\log\frac{\pi^2N^2}{3\delta}}$ is monotonically decreasing in $N$. Therefore, 
\begin{equation} 
\begin{split}
 &\max_{i,j}\left\{2\beta(N_i(t))+2\beta(N_j(t))\right\} \le 4\beta(2A).
\end{split}
\end{equation}

\section{Proof of Theorem~\ref{thm:proposed}} \label{app:thm3}
MOUCB algorithm firstly pulls each arm $2A$ times. Then for $t>2AK$ and under event $\mathcal{E}_2$, if at round $t+1$, MOUCB algorithm choose a non-optimal arm $I_{t+1}= a \not= i_O$, we have
\begin{equation} \nonumber
  \begin{split}
&\hat{\mu}_a+\beta(N_a(t))+\\
&\frac{2A}{N_a(t)} \max_{i,j}\left\{\hat{\mu}_i-\hat{\mu}_j+\beta(N_i(t)) +\beta(N_j(t)) \right\}\\
\ge& \hat{\mu}_{i_O}+\beta(N_{i_O}(t))+\\
&\frac{2A}{N_{i_O}(t)} \max_{i,j}\left\{\hat{\mu}_i-\hat{\mu}_j+\beta(N_i(t)) +\beta(N_j(t)) \right\},
\end{split}
\end{equation}
which implies to 
\begin{equation} \nonumber
  \begin{split}
&\hat{\mu}_a+\frac{A}{N_a(t)}\left( 2\Delta_{i_O,i_W}+8\sqrt{\frac{\sigma^2K}{A}\log\frac{4\pi^2 A^2}{3\delta}}\right)+\beta(N_a(t))\\
\ge& \hat{\mu}_{i_O}+\frac{A}{N_{i_O}(t)}\Delta_{i_O,i_W}+\beta(N_{i_O}(t)),
\end{split}
\end{equation}
according to Lemma~\ref{thm:banditdelta}.

From equation~\eqref{eq:l3}, we could find
\begin{equation} \nonumber
  \begin{split}
\hat{\mu}_a &\le \frac{1}{N_a(t)}\sum_{s \in \tau_a(t)} \mu_{I_s^0} + \beta(N_a(t))\\
& \le \mu_a+\Delta_{i_O,a}\frac{ C_a}{N_a(t)} + \beta(N_a(t))\\
&\le \mu_a+\Delta_{i_O,a}\frac{A}{N_a(t)} + \beta(N_a(t)),
\end{split}
\end{equation}
and
\begin{equation} \nonumber
  \begin{split}
\hat{\mu}_{i_O} &\ge \frac{1}{N_{i_O}(t)}\sum_{s \in \tau_{i_O}(t)} \mu_{I_s^0} - \beta(N_{i_O}(t))\\
& \ge \mu_{i_O}-\Delta_{i_O,i_W}\frac{ C_{i_O}}{N_{i_O}(t)} - \beta(N_{i_O}(t))\\
& \ge \mu_{i_O}-\Delta_{i_O,i_W}\frac{A}{N_{i_O}(t)} - \beta(N_{i_O}(t)).
\end{split}
\end{equation}

By combining the inequalities above, we have
\begin{equation} \nonumber
  \begin{split}
\mu_{i_O} \le& \mu_a+\Delta_{i_O,a}\frac{A}{N_a(t)} + 2\beta(N_a(t))+ \\
&\frac{A}{N_a(t)}\left( 2\Delta_{i_O,i_W}+8\sqrt{\frac{\sigma^2K}{A}\log\frac{4\pi^2 A^2}{3\delta}}\right),
\end{split}
\end{equation}
which is equivalent to 
\begin{equation} \nonumber
  \begin{split}
\Delta_{i_o,a} \le& \Delta_{i_O,a}\frac{A}{N_a(t)} + 2\sqrt{\frac{2\sigma^2K}{N_a(t)}\log\frac{\pi^2 (N_a(t))^2}{3\delta}}+ \\
&\frac{A}{N_a(t)}\left( 2\Delta_{i_O,i_W}+8\sqrt{\frac{\sigma^2K}{A}\log\frac{4\pi^2 A^2}{3\delta}}\right) \\
\le&2\sqrt{\frac{2\sigma^2K}{N_a(t)}\log\frac{\pi^2 t^2}{3\delta}}+\frac{A}{N_a(t)} \left( \Delta_{i_O,a} + \right.\\
&\left. 2\Delta_{i_O,i_W}+8\sqrt{\frac{\sigma^2K}{A}\log\frac{4\pi^2 A^2}{3\delta}}\right).
\end{split}
\end{equation}
Therefore,
\begin{equation} 
\begin{split}
N_a(t)\le \max &\left\{\frac{8\sigma^2K}{\Delta_{i_O,a}^2}\log\frac{\pi^2t^2}{3\delta}, 
\frac{A}{\Delta_{i_o,a}} \left( \Delta_{i_O,a} \right.\right.\\
&\left.\left.+ 2\Delta_{i_O,i_W} + 8\sqrt{\frac{\sigma^2K}{A}\log\frac{4\pi^2 A^2}{3\delta}}\right) \right\}. \label{eq:th3eq}
\end{split}
\end{equation}
As event $\mathcal{E}_2 $ holds with probability at least $1- \delta$, \eqref{eq:th1eq} holds with probability at least $1- \delta$. Then Theorem~\ref{thm:proposed} follows immediately from the definition of the pseudo-regret in~\eqref{eq:pregertdefine} and equation \eqref{eq:th3eq}.

\bibliographystyle{IEEEbib}
\bibliography{mybib}

\end{document}